\documentclass{dcogpaper}
\usepackage{mathtools,tikz,float,tabularx,xurl}

\DeclareMathAlphabet{\mathsf}{OT1}{lmss}{m}{n}
\SetMathAlphabet{\mathsf}{bold}{OT1}{lmss}{bx}{n}

\title{Topological Sheaves of Irreducibility: A Theoretical Framework for Multi-Agent Task Obstructions}
\author{Steven Motta\thanks{Correspondence: [first name]@distributed.inc.}\and Davide Mambelli\texorpdfstring{\textsuperscript{1}}{}\and Alessandro Rozza\texorpdfstring{\textsuperscript{1}}{}}
\affiliation{distributed cognition}
\date{2026-10-19}
\shorttitle{Topological Sheaves of Irreducibility}
\paperstatus{Preprint}
\paperlogo{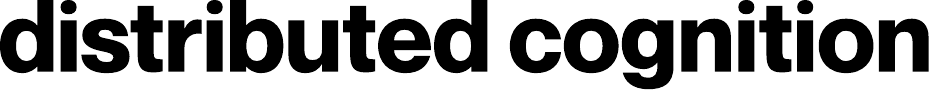}
\papersymbol{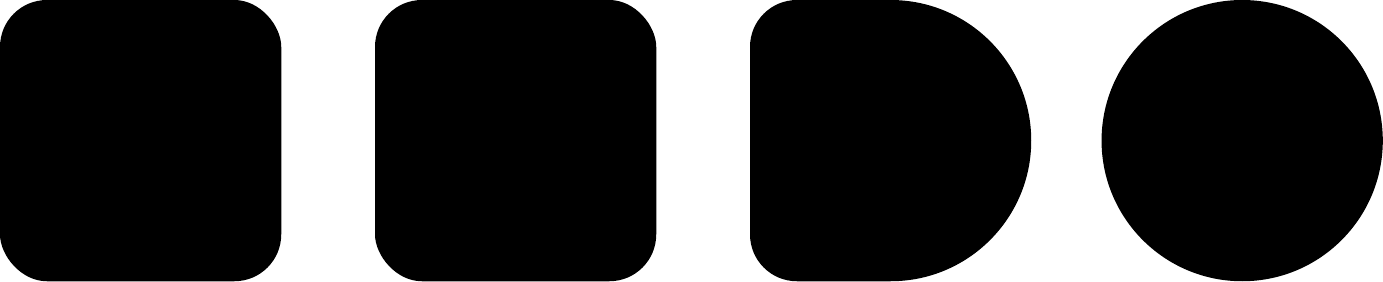}
\papercopyright{\textcopyright\ 2026 distributed cognition. All rights reserved.}

\newtheorem{Definition}{Definition}[section]
\newtheorem{Proposition}{Proposition}[section]
\newtheorem{Hypothesis}{Hypothesis}[section]
\theoremstyle{remark}
\newtheorem{Remark}{Remark}[section]

\begin{document}
\maketitle
\begin{abstract}
Multi-agent systems can differ under identical agents and communication
budgets when they support different forms of joint interaction. We introduce
\textit{Topological Sheaves of Irreducibility} (TSI), a sheaf-theoretic framework
separating task-level compatibility obstructions from system capability. A task is
represented by an interaction complex $K$ and a capacity sheaf $\mathcal{F}$; nonzero
$H^n(K;\mathcal{F})$ defines a structural signature, while registered interventions test
capability. For cocycle data, cochain realization is solvable exactly when the represented cohomology class vanishes. Examples show that the signature is sheaf-relative, and the degree-1 reduction yields cycle-space witnesses and observation requirements. We evaluate
two finite one-round models. Across $15{,}625$ degree-1 inputs, an all-agent additive primitive solves decision and repair exactly; the optimal pairwise system reaches $88.80\%$ on decision, one input above the $88.79\%$ majority-class baseline, and $26.63\%$ on repair. On a degree-2 tetrahedral-boundary task over $531{,}441$ inputs, the same-2-skeleton system retains all four three-agent primitives yet matches the pairwise-only optimum throughout the registered separation window, whereas adding the
class-filling four-agent primitive makes both tasks exact. Status-preserving robustness checks retain exact performance but are not cost-matched to the class-restoring ablation.
Conditional on the registered primitives and cost models, TSI provides task-relative, falsifiable tests of higher-order interaction.
\end{abstract}

\section{Introduction}

Multi-agent systems are often compared by changing the agents, the amount of communication, or the learning algorithm. However, two systems can use the same agents and the same communication budget and still differ because they support different \emph{forms of interaction}. Pairwise communication is not always equivalent to a genuinely joint operation under the same task and communication constraints. This distinction becomes important when a task requires locally produced information to be combined into one globally consistent solution.

A simple example illustrates the issue. Suppose three agents can exchange messages pairwise around a cycle, but cannot perform one operation whose inputs and outcome are jointly available to all three. An ordinary or weighted graph records only pairwise relations. Edge weights can express the cost, strength, or bandwidth of each pairwise channel, but they cannot represent an atomic operation available only to the three agents together. A simplicial complex distinguishes the two cases: three pairwise channels give three edges, while a genuine three-way operation adds the filled 2-simplex they bound. Cellular sheaves then specify what information each interaction carries and how local pieces must agree.

This leads to the central idea of TSI. We separate the \emph{structure of the task} from the \emph{capability of the system}. The task is represented by a pair $(K,\mathcal{F})$, where $K$ describes the task-relevant compatibility structure and $\mathcal{F}$ describes the corresponding data and restriction maps. A nonzero cohomology group $H^n(K;\mathcal{F})$ identifies a sheaf-relative obstruction at cochain degree $n$. We call this a \emph{structural TSI signature}. It is an obstruction carried by the task model, not a score assigned to the system. For the cochain-realization family studied in Section~\ref{sec:correspondence}, Proposition~\ref{prop:correspondence} gives the obstruction an exact realizability meaning in every cochain degree; Sections~\ref{sec:experiment} and~\ref{sec:deg2} instantiate degrees 1 and 2. System capability is evaluated separately by asking whether a registered interaction primitive removes the relevant obstruction and whether performance degrades when that primitive is ablated. Section~\ref{sec:operational} formalises this separation.

This viewpoint differs from several established frameworks for multi-agent synergy and irreducibility. Cooperative game theory assigns value to coalitions~\cite{shapley1953}; partial information decomposition measures information that is available only jointly~\cite{williams2010}; and integrated-information approaches study irreducibility under a different set of assumptions~\cite{tononi2004,oizumi2014}. These frameworks are useful comparison points, but they do not directly encode the combination of interaction geometry and task-specific compatibility constraints that TSI targets.

The cohomological criterion underlying TSI is the standard equivalence between a vanishing class and membership in the image of the coboundary map. We use this criterion to specify distributed tasks and then study their communication requirements. The contribution is the task--system separation, the class-targeted intervention protocol, and the finite-model comparisons under explicit primitive and cost assumptions. The paper develops these points as follows:
\begin{enumerate}
\item \textbf{A clear separation between signature and capability.} We define the structural signature using sheaf cohomology, and we define operational support using behavioural, class-targeted ablations (Definitions~\ref{def:signature} and~\ref{def:operational}).
\item \textbf{A task formulation of the cohomological criterion.} For cochain-realization and consistency-decision tasks, Proposition~\ref{prop:correspondence} records that a cocycle datum is realizable exactly when its cohomology class vanishes. For distributed degree-1 difference constraints, the corresponding cycle-space certificates give Fredholm/Farkas-type infeasibility witnesses (Proposition~\ref{prop:farkas}).
\item \textbf{An observation requirement for the registered task family.} For decision rules that factor through per-simplex observations, Proposition~\ref{prop:support} identifies which $n$-cell data an exact cochain-realizability rule must observe, with the triangle argument as the degree-1 special case.
\item \textbf{Examples showing that topology alone is not enough.} A twisted sheaf can remove a degree-1 class from a cycle, while a representational sheaf can create one on a contractible complex (Sections~\ref{sec:twisted} and~\ref{sec:representational}).
\item \textbf{A practical modelling protocol.} Section~\ref{sec:recipes} gives three recipes for constructing capacity sheaves and requires the task model to be fixed before behavioural evaluation.
\item \textbf{Fixed-task benchmarks in degrees 1 and 2.} We hold each task model fixed and vary only the system interaction structure. The degree-1 experiment gives a same-1-skeleton separation under exhaustive enumeration; the degree-2 tetrahedral-boundary experiment gives the analogous same-2-skeleton separation and adds status-preserving stress tests (Section~\ref{sec:benchmark}).
\end{enumerate}

The intended claim is therefore limited but testable. A nonzero $H^n(K;\mathcal{F})$ identifies a class of task-level compatibility demands. The empirical question is whether systems equipped with the registered interaction structure perform better specifically on tasks carrying the predicted obstruction, and whether that advantage disappears under the corresponding ablation.

\section{Background}

\textbf{Interaction topology.} A graph cannot distinguish three pairwise links from a genuine three-way interaction. A simplicial or cellular complex can: a $k$-simplex records a $(k{+}1)$-way interaction, while an unfilled cycle records a hole. Applied topology uses these objects to encode higher-order relations~\cite{ghrist2014,hatcher2002}. Related work covers higher-order networks~\cite{battiston2020,torres2021,battiston2021,bick2023}, empirical motif and simplex structure~\cite{benson2016}, and Hodge-Laplacian methods on simplicial complexes~\cite{schaub2020,barbarossa2020}. Recent studies further show that the choice of higher-order representation can itself change collective dynamics~\cite{zhang2023}, and that homological structure can be used to analyse empirically reconstructed high-order interaction networks~\cite{feng2024}. Related work combines simplicial lifting, generalized Hodge operators, and distributed decision optimisation for multilayer networks~\cite{yuan2026}. These results reinforce the need to distinguish the chosen interaction representation from the task-specific semantics that TSI places in the capacity sheaf.

\textbf{Cellular sheaves.} A cellular sheaf attaches vector spaces and compatible linear maps to the cells of a complex. Cellular descriptions of sheaf data on stratified spaces trace back to Shepard's thesis~\cite{shepard1985}, with modern computational formulations developed in subsequent work~\cite{curry2014,curry2016}. Existing applications include local measurement systems, sensor consistency, network coding, and structured graph learning~\cite{robinson2014,robinson2017,ghrist2011network,bodnar2022,hansen2019}. Discourse sheaves use restriction maps to model how agents translate shared topics through private frames~\cite{hansen2021}, which is close to the capacity semantics used here. More broadly, recent topological deep-learning work has emphasised learning and message passing on simplicial, cellular, and other non-graph domains~\cite{papamarkou2024,zia2024}. Sheaf cohomology depends on the stalks and restriction maps, so it is not a homotopy invariant of the underlying complex. That dependence lets the model encode task semantics as well as interaction shape.

Sheaf-theoretic local-to-global obstructions also play a central role in contextuality. Abramsky and Brandenburger formulate contextuality and non-locality in terms of compatible local measurement data that fail to admit a global section~\cite{abramsky2011}. This is an important conceptual antecedent for TSI. The present framework differs in using cellular capacity sheaves on registered multi-agent interaction complexes, degree-specific sheaf cohomology as a task-relative structural signature, and separate behavioural interventions to assess system capability.

Abramsky, Mansfield and Soares Barbosa subsequently made the obstruction explicitly cohomological, using \v{C}ech cohomology of an abelian presheaf derived from local sections over measurement contexts to construct cohomological witnesses for contextuality~\cite{abramsky2012cohomology}. This is the closest antecedent to the present construction: both use a non-vanishing cohomological obstruction to witness a failure of local-to-global assembly. In the contextuality setting this witness is sufficient but not necessary, whereas Proposition~\ref{prop:correspondence} gives an exact criterion for the linear cochain-realization family considered here. The settings also differ operationally: there the presheaf is built over measurement contexts, whereas TSI places cellular capacity sheaves on registered multi-agent interaction complexes and couples the structural class to system capability through separately registered interventions (Definition~\ref{def:operational}).

\textbf{Multi-agent capability and irreducibility.} Partial information decomposition studies information available only from variables taken jointly~\cite{williams2010}. Cooperative game theory studies coalition value and superadditivity~\cite{shapley1953}, and integrated-information theory uses its own axioms for irreducibility~\cite{tononi2004,oizumi2014}. TSI instead models the interaction complex and task-specific compatibility constraints. Section~\ref{sec:relation} treats the other quantities as comparison targets.

\section{Formal Framework}

\subsection{Interaction Complexes}

Let $A=\{a_1,\ldots,a_m\}$ be a finite set of agents. We model task-relevant interactions by a finite abstract simplicial complex $K$ on $A$.

\begin{Definition}[Operational simplex criterion]\label{def:opsimplex}
Let $\Pi$ be the registered communication model with its budget. A simplex $\sigma$ with $|\sigma|\geq 2$ is admissible in $K$ only if its agents can perform, under $\Pi$, a joint operation that cannot be simulated within the same budget by relaying over the proper faces of $\sigma$. Such an operation may be a simultaneous observation, an atomic joint decision, or a shared channel with common-knowledge semantics. Message traffic among the agents of $\sigma$ does not by itself make $\sigma$ admissible. A missing simplex means that no such joint operation is registered as available under $\Pi$.
\end{Definition}

\begin{Definition}[Interaction complex]\label{def:complex}
An interaction complex is an abstract simplicial complex $K$ whose vertices are agents and whose $k$-simplices represent registered $(k+1)$-way interactions relevant to a task. Thus $\{a_i,a_j\}\in K$ records a pairwise interaction, whereas $\{a_i,a_j,a_\ell\}\in K$ records a registered three-way interaction that is not reducible, under the specified communication model and budget, to relaying over its proper faces.
\end{Definition}

The construction of $K$ is part of the empirical model and must satisfy the operational simplex criterion of Definition~\ref{def:opsimplex}. In particular, the presence of a simplex records the availability of a registered interaction primitive under the stated model, rather than merely the observation of message traffic among its agents.

\subsection{Capacity Sheaves}

\begin{Definition}[Capacity sheaf]\label{def:sheaf}
Let $K$ be a finite simplicial complex. A \emph{capacity sheaf} $\mathcal{F}$ of finite-dimensional real vector spaces assigns:
\begin{enumerate}
\item a stalk $\mathcal{F}(\sigma)$ to every simplex $\sigma \in K$; and
\item a linear restriction map $\rho_{\sigma\tau} : \mathcal{F}(\sigma) \to \mathcal{F}(\tau)$ for every face relation $\sigma \subseteq \tau$,
\end{enumerate}
subject to $\rho_{\sigma\sigma} = \mathrm{id}$ and $\rho_{\tau\upsilon} \circ \rho_{\sigma\tau} = \rho_{\sigma\upsilon}$ whenever $\sigma \subseteq \tau \subseteq \upsilon$.
\end{Definition}

Here $\mathcal{F}(v)$ is the space of states, plans, estimates, or representations available to agent $v$. For $|\sigma| \geq 2$, $\mathcal{F}(\sigma)$ contains the quantities on which the agents in $\sigma$ must agree or that they must produce jointly. The map $\rho_{\sigma\tau}$ expresses a subgroup's data in the coordinates of the larger interaction. The model specifies what its vectors mean and why the maps preserve task structure. Section~\ref{sec:recipes} gives three constructions. The constant sheaf $\underline{\mathbb{R}}$, with stalk $\mathbb{R}$ and identity maps, is the topology-only baseline.

\subsection{Cochains and Cohomology}

After choosing an orientation for each simplex, define
\begin{equation}
C^k(K;\mathcal{F}) = \bigoplus_{\sigma \in K_k} \mathcal{F}(\sigma),
\end{equation}
where $K_k$ is the set of oriented $k$-simplices. The coboundary maps $\delta^k : C^k(K;\mathcal{F}) \to C^{k+1}(K;\mathcal{F})$ are induced by the restriction maps and signed incidences: for an oriented $(k{+}1)$-simplex $\tau = [v_0, \dots, v_{k+1}]$,
\begin{equation}
(\delta^k f)(\tau) = \sum_{i=0}^{k+1} (-1)^i \, \rho_{d_i\tau,\,\tau}\big(f(d_i\tau)\big), \qquad d_i\tau = [v_0, \dots, \widehat{v_i}, \dots, v_{k+1}].
\end{equation}
The functoriality of $\mathcal{F}$ gives $\delta^{k+1} \circ \delta^k = 0$, yielding the cochain complex
\begin{equation}
0 \longrightarrow C^0(K;\mathcal{F}) \xrightarrow{\;\delta^0\;} C^1(K;\mathcal{F}) \xrightarrow{\;\delta^1\;} C^2(K;\mathcal{F}) \xrightarrow{\;\delta^2\;} \cdots,
\end{equation}
whose cohomology is
\begin{equation}
H^n(K;\mathcal{F}) = \ker(\delta^n) \big/ \operatorname{im}(\delta^{n-1}).
\end{equation}
A class in $H^n(K;\mathcal{F})$ is an $n$-dimensional compatible cochain that is not the coboundary of lower-dimensional data.

\begin{Definition}[Structural TSI signature]\label{def:signature}
A task model $(K, \mathcal{F})$ has a \emph{structural TSI signature of degree $n$} when
\begin{equation}
H^n(K;\mathcal{F}) \neq 0.
\end{equation}
The dimension $\beta^n(K;\mathcal{F}) = \dim H^n(K;\mathcal{F})$ is its multiplicity.
\end{Definition}

\begin{Definition}[Task-relative irreducibility]\label{def:irreducibility}
A cocycle demand $c\in\ker(\delta^n)\subseteq C^n(K;\mathcal{F})$ is \emph{reducible} in $(K,\mathcal{F})$ if
$c\in\operatorname{im}(\delta^{n-1})$, and \emph{irreducible} otherwise. Equivalently,
\[
c \text{ is irreducible}
\quad\Longleftrightarrow\quad
[c]\neq 0 \text{ in } H^n(K;\mathcal{F}).
\]
\end{Definition}
A structural TSI signature of degree $n$ therefore exists precisely when the task model admits at least one irreducible cocycle demand at that degree. Irreducibility in this sense is a property of the pair $(K,\mathcal{F})$ and the cochain degree, not of the agents or their intelligence: it states that the cocycle demand cannot be assembled from $(n{-}1)$-dimensional data over the registered interaction structure. It is distinct from irreducibility in cooperative game theory or integrated information theory; Section~\ref{sec:relation} treats those notions as comparison targets rather than equivalents.

\begin{Remark}[What the signature is not]\label{rem:notintel}
The signature is a property of the pair $(K,\mathcal{F})$, not of the complex alone (Sections~\ref{sec:twisted} and~\ref{sec:representational}), and not of the system's intelligence. Its degree identifies an obstruction \emph{type}, not a ranking; $\beta^n$ should not be compared across tasks without normalisation and a behavioural interpretation.
\end{Remark}

\subsection{Operational Support for TSI Signatures}\label{sec:operational}

Operational support couples the structural signature to system performance and controlled interventions.

\begin{Definition}[Operational support]\label{def:operational}
Let $\mathcal{T}$ be a task family with a registered model $(K,\mathcal{F})$ carrying a structural signature of degree $n$ with class $[\omega] \neq 0$. An \emph{admissible intervention} has a registered action on the interaction complex and one of the following comparison maps.
\begin{enumerate}
\item \textbf{Removal.} Removing simplices gives a subcomplex $\iota:K^-\hookrightarrow K$ and the restriction map
\[
\iota^*:H^n(K;\mathcal{F})\longrightarrow H^n(K^-;\mathcal{F}|_{K^-}).
\]
\item \textbf{Sheaf alteration.} A sheaf morphism $\varphi:\mathcal{F}\to\mathcal{F}'$ over $K$ gives
\[
\varphi_*:H^n(K;\mathcal{F})\longrightarrow H^n(K;\mathcal{F}').
\]
\item \textbf{Addition.} Adding simplices gives an inclusion $\jmath:K\hookrightarrow K^+$ together with a \emph{registered extension} $\widetilde{\mathcal{F}}$ on $K^+$ satisfying $\jmath^{-1}\widetilde{\mathcal{F}}\cong\mathcal{F}$. The induced restriction map is
\[
\jmath^*:H^n(K^+;\widetilde{\mathcal{F}})\longrightarrow H^n(K;\mathcal{F}).
\]
\end{enumerate}
For removal and sheaf alteration, the corresponding map records whether $[\omega]$ survives or is sent to zero. For an addition, no natural map sends $[\omega]$ forward. We therefore call the addition \emph{class-filling} when
\[
[\omega]\notin\operatorname{im}(\jmath^*),
\]
namely when the task class does not extend to the enlarged complex; it is \emph{extension-preserving} when $[\omega]\in\operatorname{im}(\jmath^*)$.

A joint primitive on a participant set $\sigma$ is interpreted relative to a registered base complex $K_0$ through its \emph{primitive closure}
\[
\mathsf{P}_{\sigma}(K_0)
 = \{\tau\subseteq\sigma:\ |\tau|\geq 2,\ \tau\notin K_0\}.
\]
Ablating the primitive removes its entire closure $\mathsf{P}_{\sigma}(K_0)$ and retains the simplices already present in $K_0$. Removing only the maximal simplex is not an admissible primitive ablation. The ablation is \emph{class-restoring} if $[\omega]$ extends again after removal. Controls are matched by registered primitive cost, including external input/output width, per-agent congestion, parameter count, and retained base complex. The raw number of removed simplices is not the matching criterion.

A multi-agent system $S$ \emph{provides operational support for the degree-$n$ TSI signature on $\mathcal{T}$} if:
\begin{enumerate}
\item \textbf{(Performance)} $S$ solves $\mathcal{T}$ above a pre-specified threshold under a fixed communication budget;
\item \textbf{(Class-restoring sensitivity)} for a registered class-filling mechanism used by $S$, primitive-closure ablations that restore $[\omega]$ degrade performance on $\mathcal{T}$ by more than a pre-specified margin; and
\item \textbf{(Status-preserving insensitivity)} matched interventions of equal registered cost that leave the extension status of $[\omega]$ unchanged do \emph{not} produce a comparable degradation.
\end{enumerate}
\end{Definition}

All three conditions are behavioural and falsifiable. Condition (1) without (2) and (3) is an ordinary benchmark result; (2) without (3) cannot distinguish a topology-specific mechanism from generic capacity loss. Operational support is the conjunction.

\begin{Remark}[Why the comparison maps are needed]\label{rem:comparison}
An intervention changes $(K,\mathcal{F})$ and therefore changes the group $H^n$. Comparing a class before and after the intervention requires a canonical map. For additions, this map also requires an explicit sheaf extension $\widetilde{\mathcal{F}}$ on $K^+$; the expression $H^n(K^+;\mathcal{F})$ is otherwise undefined. For example, under $\jmath:\partial\Delta^2\hookrightarrow\Delta^2$ with constant coefficients,
\[
H^1(\Delta^2;\underline{\mathbb{R}})=0\longrightarrow H^1(\partial\Delta^2;\underline{\mathbb{R}})\cong\mathbb{R},
\]
so the boundary generator is not in the image of $\jmath^*$: a 1-cochain with nonzero circulation ceases to be a cocycle once the 2-simplex is present. In relative cellular sheaf cohomology, the long exact sequence of $(K^+,K)$ contains
\[
H^n(K^+;\widetilde{\mathcal{F}})\xrightarrow{\ \jmath^*\ }H^n(K;\mathcal{F})
\longrightarrow H^{n+1}(K^+,K;\widetilde{\mathcal{F}}),
\]
so $[\omega]$ extends exactly when its relative obstruction vanishes. A primitive-level ablation must therefore remove the full closure generated by the primitive. Removing only its maximal simplex would leave some of the primitive's lower faces available.
\end{Remark}

\begin{Remark}[Signature--capability separation]\label{rem:separation}
Definition~\ref{def:operational} keeps the two roles separate. The class $[\omega]$ describes the task's compatibility demand, as made explicit by Proposition~\ref{prop:correspondence}. System capability is the observed ability to meet that demand, either through interaction structure that fills the obstruction or through an aggregation strategy subject to the observation requirements that Proposition~\ref{prop:support} imposes on rules factoring through per-simplex observations.
\end{Remark}

\section{Worked Examples}

\subsection{The Three-Agent Cycle: Topology Creates a Class}\label{sec:triangle}

Consider agents $0, 1, 2$ with pairwise channels but no three-way channel: $K = \partial\Delta^2$, with oriented edges $[0,1]$, $[0,2]$, $[1,2]$ and no 2-simplex. Equip $K$ with the constant sheaf $\underline{\mathbb{R}}$. Then $C^0 \cong \mathbb{R}^3$, $C^1 \cong \mathbb{R}^3$, $C^2 = 0$, and
\begin{equation}
\delta^0 = \begin{pmatrix} -1 & 1 & 0 \\ -1 & 0 & 1 \\ 0 & -1 & 1 \end{pmatrix}.
\end{equation}
For $c = (c_{01}, c_{02}, c_{12}) \in C^1$, define $\ell(c) = c_{01} - c_{02} + c_{12}$. A direct check gives $\ell \circ \delta^0 = 0$, so $\operatorname{im}(\delta^0) \subseteq \ker(\ell)$; both spaces have dimension two, hence they are equal, and since $C^2 = 0$,
\begin{equation}
H^1(K;\underline{\mathbb{R}}) \cong \mathbb{R}^3 \big/ \operatorname{im}(\delta^0) \cong \mathbb{R}.
\end{equation}

Figure~\ref{fig:classfilling} summarizes the class-filling mechanism in degrees~1 and~2 for the constant sheaf.

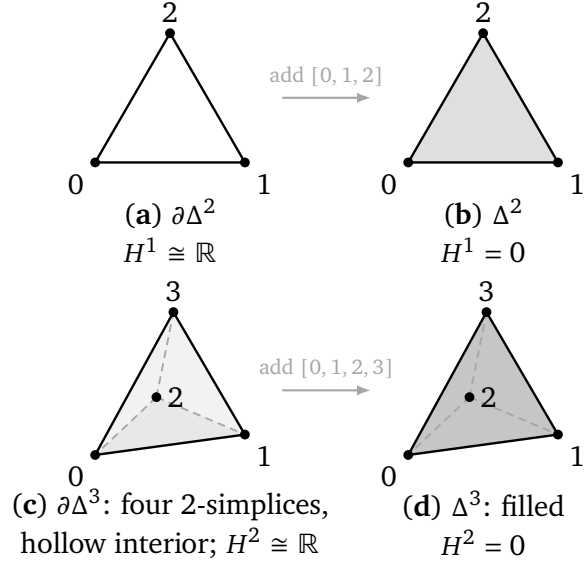
\begin{figure}[H]
\centering
\begin{tikzpicture}[
  scale=0.9,
  vtx/.style={circle, fill=black, inner sep=1.3pt},
  edg/.style={line width=0.9pt},
  hid/.style={line width=0.7pt, densely dashed, gray!70},
  fil/.style={fill=gray!25},
  sol/.style={fill=gray!45},
  lbl/.style={font=\small},
  panelcap/.style={font=\small, align=center},
  arr/.style={-latex, line width=0.8pt, gray!70}
]

%%%%%%%%%%%%%%%%%%%% (a) boundary of Delta^2 %%%%%%%%%%%%%%%%%%%%
\begin{scope}[shift={(0,0)}]
  \coordinate (a0) at (0,0);
  \coordinate (a1) at (2.2,0);
  \coordinate (a2) at (1.1,1.9);
  \draw[edg] (a0)--(a1)--(a2)--cycle;
  \node[vtx] at (a0){}; \node[vtx] at (a1){}; \node[vtx] at (a2){};
  \node[lbl,below left]  at (a0){$0$};
  \node[lbl,below right] at (a1){$1$};
  \node[lbl,above]       at (a2){$2$};
  \node[panelcap] at (1.1,-1.0)
    {(\textbf{a})~$\partial\Delta^{2}$\\[1pt]$H^{1}\cong\mathbb{R}$};
\end{scope}

\draw[arr] (2.75,0.95) -- (4.05,0.95)
  node[midway, above, font=\scriptsize]{add $[0,1,2]$};

%%%%%%%%%%%%%%%%%%%% (b) Delta^2 %%%%%%%%%%%%%%%%%%%%
\begin{scope}[shift={(4.6,0)}]
  \coordinate (b0) at (0,0);
  \coordinate (b1) at (2.2,0);
  \coordinate (b2) at (1.1,1.9);
  \fill[fil] (b0)--(b1)--(b2)--cycle;
  \draw[edg] (b0)--(b1)--(b2)--cycle;
  \node[vtx] at (b0){}; \node[vtx] at (b1){}; \node[vtx] at (b2){};
  \node[lbl,below left]  at (b0){$0$};
  \node[lbl,below right] at (b1){$1$};
  \node[lbl,above]       at (b2){$2$};
  \node[panelcap] at (1.1,-1.0)
    {(\textbf{b})~$\Delta^{2}$\\[1pt]$H^{1}=0$};
\end{scope}

%%%%%%%%%%%%%%%%%%%% (c) boundary of Delta^3 %%%%%%%%%%%%%%%%%%%%
\begin{scope}[shift={(0,-4.3)}]
  \coordinate (c0) at (0,0);
  \coordinate (c1) at (2.2,0.30);
  \coordinate (c2) at (0.90,0.85);   % rear vertex
  \coordinate (c3) at (1.15,2.10);   % apex
  \fill[gray!18] (c0)--(c1)--(c3)--cycle;
  \fill[gray!10] (c0)--(c2)--(c3)--cycle;
  \fill[gray!10] (c1)--(c2)--(c3)--cycle;
  \draw[hid] (c2)--(c0); \draw[hid] (c2)--(c1); \draw[hid] (c2)--(c3);
  \draw[edg] (c0)--(c1)--(c3)--cycle;
  \node[vtx] at (c0){}; \node[vtx] at (c1){};
  \node[vtx] at (c2){}; \node[vtx] at (c3){};
  \node[lbl,below left]  at (c0){$0$};
  \node[lbl,below right] at (c1){$1$};
  \node[lbl,right]       at (c2){$2$};
  \node[lbl,above]       at (c3){$3$};
  \node[panelcap] at (1.1,-1.0)
    {(\textbf{c})~$\partial\Delta^{3}$: four 2-simplices,\\[1pt]
     hollow interior; $H^{2}\cong\mathbb{R}$};
\end{scope}

\draw[arr] (2.75,-3.35) -- (4.05,-3.35)
  node[midway, above, font=\scriptsize]{add $[0,1,2,3]$};

%%%%%%%%%%%%%%%%%%%% (d) Delta^3 %%%%%%%%%%%%%%%%%%%%
\begin{scope}[shift={(4.6,-4.3)}]
  \coordinate (d0) at (0,0);
  \coordinate (d1) at (2.2,0.30);
  \coordinate (d2) at (0.90,0.85);
  \coordinate (d3) at (1.15,2.10);
  \fill[sol] (d0)--(d1)--(d3)--cycle;
  \draw[hid] (d2)--(d0); \draw[hid] (d2)--(d1); \draw[hid] (d2)--(d3);
  \draw[edg] (d0)--(d1)--(d3)--cycle;
  \node[vtx] at (d0){}; \node[vtx] at (d1){};
  \node[vtx] at (d2){}; \node[vtx] at (d3){};
  \node[lbl,below left]  at (d0){$0$};
  \node[lbl,below right] at (d1){$1$};
  \node[lbl,right]       at (d2){$2$};
  \node[lbl,above]       at (d3){$3$};
  \node[panelcap] at (1.1,-1.0)
    {(\textbf{d})~$\Delta^{3}$: filled\\[1pt]$H^{2}=0$};
\end{scope}

\end{tikzpicture}
\caption{Class-filling in degrees 1 and 2, with the constant sheaf
$\underline{\mathbb{R}}$. (\textbf{a}) Three agents with pairwise channels only
span the boundary of a triangle and carry a degree-1 class;
(\textbf{b}) the three-agent joint operation fills the 2-simplex and the class
vanishes. (\textbf{c}) The tetrahedral boundary retains all four three-agent
operations, yet carries a degree-2 class; (\textbf{d}) only the four-agent
operation fills the 3-simplex and removes it. Dashed segments are the edges
hidden behind the body. In each row the communication graph is identical on
both sides: the difference is the top-dimensional joint operation, not the
number of pairwise channels. Panel (\textbf{c}) is the configuration of the
system $S_{\mathrm{face}}$ in Section~\ref{sec:deg2}, whose exhaustive optimum
coincides with pairwise-only communication throughout the registered separation window.\label{fig:classfilling}}
\end{figure}

The model has a degree-1 structural signature. Adding the 2-simplex $[0,1,2]$ yields $\delta^1(c) = c_{01} - c_{02} + c_{12} = \ell(c)$, so $\ker(\delta^1) = \operatorname{im}(\delta^0)$ and $H^1(\Delta^2;\underline{\mathbb{R}}) = 0$: the joint channel fills the obstruction. Section~\ref{sec:correspondence} gives this example its exact operational meaning.

\subsection{A Twisted Sheaf: Sheaf Structure Destroys the Class}\label{sec:twisted}

Keep $K = \partial\Delta^2$ but replace the constant sheaf with a rank-one \textit{monodromy sheaf} (local system) $\mathcal{F}_\lambda$: all stalks $\mathbb{R}$, all restriction maps the identity except $\rho_{v_0,[0,1]}(x) = \lambda x$ for a fixed $\lambda \in \mathbb{R} \setminus \{0,1\}$. (Invertibility of the restriction maps, $\lambda \neq 0$, is what licenses the term ``monodromy''; the cohomology computation below in fact only uses $\lambda \neq 1$.) Then
\begin{equation}
\delta^0 = \begin{pmatrix} -\lambda & 1 & 0 \\ -1 & 0 & 1 \\ 0 & -1 & 1 \end{pmatrix}, \qquad \det(\delta^0) = 1 - \lambda \neq 0,
\end{equation}
so $\delta^0$ is an isomorphism, giving $H^0(K;\mathcal{F}_\lambda) = 0$ and $H^1(K;\mathcal{F}_\lambda) = 0$. The cycle is unchanged, but its signature vanishes. Here $H^0 = 0$ means that the twisted sheaf has no nonzero global sections. Since $\delta^0$ is invertible, every edge datum $c$ has a unique realizing vertex assignment. One interpretation is a systematic rescaling by agent 0 on the $[0,1]$ channel, such as a calibration or unit conversion. This changes the consistency problem: realization is uniquely solvable for every datum. Thus monodromy around the loop can remove the compatibility obstruction.

\subsection{A Representational Sheaf: Obstruction Without Topology}\label{sec:representational}

Now take $K$ to be a single edge $\tau = [0,1]$, a contractible complex with trivial ordinary cohomology. Define $\mathcal{F}$ by $\mathcal{F}(v_0) = \mathcal{F}(v_1) = \mathbb{R}$, $\mathcal{F}(\tau) = \mathbb{R}^2$, and
\begin{equation}
\rho_{v_0\tau}(x) = (x, 0), \qquad \rho_{v_1\tau}(y) = (y, 0).
\end{equation}
Then $\delta^0(f_0, f_1) = \rho_{v_1\tau}(f_1) - \rho_{v_0\tau}(f_0) = (f_1 - f_0,\, 0)$, so $\operatorname{im}(\delta^0) = \mathbb{R} \times \{0\}$ and
\begin{equation}
H^1(K;\mathcal{F}) \cong \mathbb{R}^2 \big/ \left(\mathbb{R} \times \{0\}\right) \cong \mathbb{R} \neq 0.
\end{equation}
The complex is contractible, yet it has a degree-1 class. The second component of $\mathcal{F}(\tau)$ can represent a jointly produced quantity, such as a cross-check or negotiated parameter, that neither agent can supply from its individual representation. This obstruction comes from the sheaf rather than the topology of $K$.

\begin{Remark}\label{rem:sheafrelative}
Together, the two examples show that a structural TSI signature is not simply a Betti number. Changing the sheaf can remove a class from a cycle or create one on a contractible complex. Empirical work must therefore specify and fix the pair $(K,\mathcal{F})$ before evaluation (Section~\ref{sec:recipes}).
\end{Remark}

\section{The Obstruction--Task Correspondence}\label{sec:correspondence}

To make the signature testable, we relate it to a task family for which the cohomology class determines solvability.

\subsection{Realization and Consistency-Decision Tasks}

Fix a model $(K, \mathcal{F})$. Each edge $e = [i,j]$ carries a datum $c_e \in \mathcal{F}(e)$, such as an observed offset or required relation. The \emph{initial allocation} of these data among agents is part of the registered communication model, not part of the cohomological solvability statement. The correspondence below therefore makes no assumption about which agents initially observe each $c_e$; Section~\ref{sec:benchmark} specifies concrete allocations for the one-round benchmarks.
\begin{itemize}
\item \textbf{Realization task $\mathsf{R}(c)$:} the agents must output local states $f_v \in \mathcal{F}(v)$ such that $\rho_{v_j e}(f_j) - \rho_{v_i e}(f_i) = c_e$ for every oriented edge $e = [v_i, v_j]$; i.e.\ $\delta^0 f = c$.
\item \textbf{Consistency-decision task $\mathsf{D}(c)$:} the agents must output the single bit $[\![\, c \in \operatorname{im}(\delta^0) \,]\!]$.
\end{itemize}
These degree-1 tasks cover familiar distributed problems, including clock synchronisation from pairwise offsets, sensor calibration from pairwise discrepancies, consistency of exchange rates in log coordinates, and feasibility checking for distributed constraints.

More generally, for $n\geq1$ and $c\in C^n(K;\mathcal{F})$, define the \emph{cochain-realization task} $\mathsf{R}^n(c)$ as finding $x\in C^{n-1}(K;\mathcal{F})$ such that $\delta^{n-1}x=c$, and define $\mathsf{D}^n(c)$ as deciding membership in $\operatorname{im}(\delta^{n-1})$. Thus $\mathsf{R}^1=\mathsf{R}$ and $\mathsf{D}^1=\mathsf{D}$. The task-specific interpretation of the $(n-1)$-cochain $x$ must be supplied by the application; Section~\ref{sec:deg2} gives one degree-2 construction.

\begin{Proposition}[Cochain-realization correspondence]\label{prop:correspondence}
Let $(K,\mathcal{F})$ be arbitrary, let $n\geq1$, and let $c\in C^n(K;\mathcal{F})$. Every realizable datum $c\in\operatorname{im}(\delta^{n-1})$ is a cocycle. Moreover, for every $c\in\ker(\delta^n)$, the realization task $\mathsf{R}^n(c)$ is solvable if and only if the cohomology class $[c]$ vanishes in $H^n(K;\mathcal{F})$. If $K$ is $n$-dimensional, $C^{n+1}=0$ makes every $n$-cochain a cocycle. Consequently: (i) if $H^n(K;\mathcal{F})=0$, every cocycle datum is realizable; (ii) if $H^n(K;\mathcal{F})\neq0$ then, for any norm $\|\cdot\|$ on the finite-dimensional space $C^n(K;\mathcal{F})$ and any $\epsilon>0$, every realizable datum has a non-realizable cocycle within distance $\epsilon$; and on cocycles $\mathsf{D}^n(c)=[\![\,[c]=0\,]\!]$.
\end{Proposition}

\begin{proof}
$\mathsf{R}^n(c)$ is solvable if and only if $c\in\operatorname{im}(\delta^{n-1})$. Since $\delta^n\circ\delta^{n-1}=0$, every realizable datum is a cocycle. For $c\in\ker(\delta^n)$, membership in $\operatorname{im}(\delta^{n-1})$ is equivalent to $[c]=0$ in $\ker(\delta^n)/\operatorname{im}(\delta^{n-1})=H^n(K;\mathcal{F})$. If $K$ is $n$-dimensional, then $C^{n+1}=0$. For (ii), fix any norm $\|\cdot\|$ on $C^n(K;\mathcal{F})$; since this space is finite-dimensional all norms on it are equivalent, so the statement does not depend on the choice. Choose a cocycle $\omega$ with $[\omega]\neq0$. For any realizable $c$ and any $t\neq0$, the cochain $c+t\omega$ is a cocycle with $[c+t\omega]=t[\omega]\neq0$, hence not realizable, and $\|(c+t\omega)-c\|=|t|\,\|\omega\|<\epsilon$ once $|t|<\epsilon/\|\omega\|$.
\end{proof}

The proposition gives the same algebraic realizability criterion in every cochain degree. For cocycle data, a non-realizable demand is therefore irreducible in the sense of Definition~\ref{def:irreducibility}. In degree 1, the unknowns are agent states and the criterion has the direct distributed interpretation used below; $\beta^1(K;\mathcal{F})$ counts independent directions of unrealizable edge demand. Section~\ref{sec:deg2} supplies a degree-2 instance in which the unknowns live on edges and the prescribed defects live on triangles.

\subsection{A Worked Reduction: Distributed Difference-Constraint Feasibility}\label{sec:farkas}

Consider a distributed \emph{difference-constraint system}. Agents hold variables $x_v \in \mathbb{R}$, and each edge $e = [i,j]$ of a constraint graph $K$ carries the equation $x_j - x_i = c_e$. The same form appears in clock synchronisation, sensor calibration, no-arbitrage checks on log exchange rates, and equality constraints in scheduling. For the constant sheaf, feasibility is the realization task $\delta^0 x = c$. Related sheaf-based distributed-optimization formulations impose more general homological constraints through sheaf Laplacians~\cite{hansen2019distributed}; here we focus instead on exact feasibility of difference constraints and the corresponding dual infeasibility witnesses. The reduction also identifies the meaning of the dual objects.

\begin{Proposition}[Certificate space and Fredholm/Farkas-type witnesses]\label{prop:farkas}
Let $K$ be a connected graph, $\underline{\mathbb{R}}$ the constant sheaf, and define $\mathcal{A} = (\operatorname{im}\delta^0)^{\perp} \subseteq C^1(K;\underline{\mathbb{R}})^*$. Then:
\begin{enumerate}
\item the system $\delta^0 x = c$ is feasible iff $\eta(c) = 0$ for every $\eta \in \mathcal{A}$;
\item $\mathcal{A} = \ker\big((\delta^0)^{\!\top}\big)$ 
is identified with the \emph{cycle space} under the standard Euclidean pairing of $K$, of dimension $\beta^1(K;\underline{\mathbb{R}}) = |E| - |V| + 1$; for a fixed datum $c$, the Fredholm infeasibility witnesses are precisely the functionals $\eta\in\mathcal{A}$ satisfying $\eta(c)\neq0$. These are the equality-constraint analogue of Farkas certificates;
\item if the system is infeasible, every support-minimal witness is a nonzero scalar multiple of the signed indicator of a simple cycle $\gamma$ with nonzero signed discrepancy $\sum_{e\in\gamma}\pm c_e$~\cite{rockafellar1969}. Conversely, every such cycle provides a certificate, and its constraints form a minimally infeasible subsystem.
\end{enumerate}
\end{Proposition}

\begin{proof}
(1) is Proposition~\ref{prop:correspondence} restated through the annihilator, since $\operatorname{im}(\delta^0)$ is the common kernel of $\mathcal{A}$. (2): identifying $C^1$ with its dual by the standard basis, $\mathcal{A} = (\operatorname{im}\delta^0)^{\perp} = \ker((\delta^0)^{\!\top})$, and $(\delta^0)^{\!\top}$ is the boundary map sending an edge vector to the difference of its endpoint indicators; its kernel is the cycle space, whose dimension is $|E| - |V| + 1$ for connected $K$. The finite-dimensional Fredholm alternative states that $Ax=b$ is infeasible iff there is $y$ with $y^{\!\top}A=0$ and $y^{\!\top}b\neq0$; with $A=\delta^0$, these are exactly the elements $\eta\in\mathcal{A}$ with $\eta(c)\neq0$. (3): recall that the \emph{elementary vectors} of a subspace of $\mathbb{R}^E$ are its nonzero vectors of inclusion-minimal support, and that every vector $\eta$ of the subspace admits a \emph{conformal decomposition} $\eta=\sum_{r}\eta_r$ into elementary vectors of the same subspace with $\operatorname{supp}(\eta_r)\subseteq\operatorname{supp}(\eta)$ and no sign cancellation on any coordinate; moreover, if $\eta$ is not itself a scalar multiple of an elementary vector, then each $\eta_r$ has support \emph{strictly} contained in $\operatorname{supp}(\eta)$~\cite{rockafellar1969}. For the cycle space of a graph the elementary vectors are exactly the signed indicators of simple cycles.

Now let $\eta$ be a witness of inclusion-minimal support, that is, $\eta\in\mathcal{A}$ with $\eta(c)\neq0$ and no $\eta'\in\mathcal{A}$ satisfying $\eta'(c)\neq0$ has support strictly contained in $\operatorname{supp}(\eta)$. Decompose $\eta=\sum_r\eta_r$ conformally. From $\sum_r\eta_r(c)=\eta(c)\neq0$ some summand satisfies $\eta_r(c)\neq0$, so that $\eta_r$ is itself a witness with $\operatorname{supp}(\eta_r)\subseteq\operatorname{supp}(\eta)$. Minimality of $\eta$ forbids strict containment, so $\eta$ is a scalar multiple of an elementary vector, i.e.\ of the signed indicator of a simple cycle $\gamma$, and $\eta(c)\neq0$ says exactly that $\sum_{e\in\gamma}\pm c_e\neq0$. Conversely, the signed indicator of any simple cycle lies in the cycle space, so any simple cycle with nonzero signed discrepancy is a certificate. Finally, deleting any single edge constraint from such a cycle leaves a forest system, which is always feasible; hence the constraints indexed by $\gamma$ form a minimally infeasible subsystem.
\end{proof}

For difference constraints, the cohomological obstruction space is the certificate space, and $\beta^1$ counts its independent directions. For a fixed datum, the active Fredholm/Farkas-type witnesses are the elements that pair nontrivially with it. Harmonic representatives can describe how a certificate class is distributed across cells once inner products are chosen, but they need not be sparse or local. Any localisation claim requires an additional sparsity criterion. Other reductions should provide the same three ingredients: a task family, a registered sheaf, and a dictionary between dual cohomological data and task-specific witnesses.

\subsection{What Local Observation Cannot Decide}

\begin{Proposition}[Triangle case]\label{prop:triangle}
Let $K = \partial\Delta^2$ with the constant sheaf, and let $D : C^1(K;\underline{\mathbb{R}}) \to \{0,1\}$ decide membership in $\operatorname{im}(\delta^0)$. No decision rule that observes only a proper subset of the three edge coordinates computes $D$ on all inputs.
\end{Proposition}

\begin{proof}
Let $S \subsetneq \{01, 02, 12\}$ and let $e \notin S$ be unobserved. Let $\mathbf{1}_e \in C^1$ be the indicator cochain of $e$ (value $1$ on $e$, $0$ elsewhere). The inputs $0$ and $\mathbf{1}_e$ agree on all coordinates in $S$, yet $0 \in \operatorname{im}(\delta^0)$ while $\ell(\mathbf{1}_e) = \pm 1 \neq 0$, so $\mathbf{1}_e \notin \operatorname{im}(\delta^0)$. Any rule depending only on $S$ answers identically on inputs with different labels.
\end{proof}

\begin{Proposition}[General support lower bound]\label{prop:support}
Let $(K,\mathcal{F})$ be arbitrary and fix $n\geq1$. Let $\mathcal{A}_n=(\operatorname{im}\delta^{n-1})^{\perp}\subseteq C^n(K;\mathcal{F})^*$ be the annihilator of the coboundary image. For an $n$-simplex $\sigma$ write $\iota_\sigma:\mathcal{F}(\sigma)\hookrightarrow C^n(K;\mathcal{F})$ for the inclusion of the corresponding summand, put $\eta|_\sigma:=\eta\circ\iota_\sigma\in\mathcal{F}(\sigma)^*$, and set
\begin{equation}
\begin{aligned}
W_\sigma
&= \operatorname{span}\{\,\eta|_\sigma : \eta\in\mathcal{A}_n\,\}
   \subseteq \mathcal{F}(\sigma)^*,\\
E_n^*
&= \bigcup_{\eta\in\mathcal{A}_n\setminus\{0\}}\operatorname{supp}(\eta)\\
&= \{\sigma\in K_n : W_\sigma\neq0\}.
\end{aligned}
\end{equation}
Suppose an exact decision rule for membership in $\operatorname{im}(\delta^{n-1})$ depends on the datum only through per-simplex observations $c_\sigma\mapsto O_\sigma(c_\sigma)$, and that each $O_\sigma$ is constant on the cosets of a subspace $V_\sigma\subseteq\mathcal{F}(\sigma)$ of ignored directions. Then $V_\sigma\subseteq W_\sigma^{\circ}$ for every $\sigma$, where $W_\sigma^{\circ}=\{y\in\mathcal{F}(\sigma):\phi(y)=0\ \text{for all }\phi\in W_\sigma\}$. Thus every globally ignored linear direction must lie in $W_\sigma^{\circ}$; in particular no simplex of $E_n^*$ may be ignored entirely.
\end{Proposition}

\begin{proof}
Since $\operatorname{im}(\delta^{n-1})$ is a subspace, $c\in\operatorname{im}(\delta^{n-1})$ if and only if $\eta(c)=0$ for all $\eta\in\mathcal{A}_n$. Suppose $z\in V_\sigma\setminus W_\sigma^{\circ}$ for some $\sigma$. Then some $\phi\in W_\sigma$ has $\phi(z)\neq0$, hence some $\eta\in\mathcal{A}_n$ has $\eta(\iota_\sigma z)\neq0$. Compare the two data $0$ and $\iota_\sigma z$. They agree outside $\sigma$, and on $\sigma$ they differ by $z\in V_\sigma$, along which $O_\sigma$ is constant; the rule therefore returns the same answer on both. Yet $0\in\operatorname{im}(\delta^{n-1})$ while $\eta(\iota_\sigma z)\neq0$ gives $\iota_\sigma z\notin\operatorname{im}(\delta^{n-1})$, so the rule errs on one of them. Hence $V_\sigma\subseteq W_\sigma^{\circ}$. In particular, any linear direction that the observation discards must be annihilated by every restricted witness in $W_\sigma$. Choosing a basis of $\mathcal{F}(\sigma)$ adapted to a decomposition $\mathcal{F}(\sigma)=W_\sigma^{\circ}\oplus U$ gives the coordinatewise reading: no nonzero direction in the complementary subspace $U$ can be ignored by an exact rule.
\end{proof}

For the $m$-cycle with constant sheaf and $n=1$, $\mathcal{A}_1$ is spanned by the full-support functional $\ell_m(c)=\sum_i\pm\,c_{i,i+1}$. Thus $E_1^*$ contains every edge, and any exact decision procedure must depend on each edge datum. This is an observation lower bound: it specifies which inputs the decision may ignore. If edge data are initially local, information from every edge must reach the decision rule; a joint evaluation can aggregate those data in one operation. A round or bit lower bound is a different statement and needs additional assumptions about synchrony, output location, precision, randomisation, and message size; a protocol may aggregate or compress the raw cell values while still depending on all of them. We do not derive such a bound at this level of generality. Sections~\ref{sec:experiment} and~\ref{sec:deg2} prove one-round bottlenecks for two specific registered coordinator-output models, and the combinatorial topology of distributed computing~\cite{herlihy2013} provides tools for the general analysis.

\section{Structural Properties}

\subsection{Invariance Is Sheaf-Relative}

For the constant sheaf, $H^n(K;\underline{\mathbb{R}})$ is ordinary cohomology and is invariant under homotopy equivalence. The boundaries of a triangle and a square therefore carry the same degree-1 signature. For a general capacity sheaf, the shape of $K$ does not determine cohomology (Sections~\ref{sec:twisted} and~\ref{sec:representational}); the relevant object is the pair $(K,\mathcal{F})$ up to sheaf-preserving equivalence. Empirical reports must specify both.

\subsection{Degree Is an Obstruction Type, Not a Ranking}

For the constant sheaf, degree has a geometric interpretation: degree 1 corresponds to cycle-level incompatibility, degree 2 to a two-dimensional cavity, and so on. This interpretation fails for a general capacity sheaf. Section~\ref{sec:representational}, for example, gives a degree-1 class on a contractible complex. In that setting, degree identifies the cochain level at which a compatibility pattern fails to arise from lower-dimensional data. Multiplicities $\beta^n$ also require task-specific normalisation and behavioural interpretation. Proposition~\ref{prop:correspondence} gives a generic algebraic cochain-realization interpretation in every degree, while domain-specific semantics remain task dependent; Section~\ref{sec:deg2} gives one explicit degree-2 instance.

\section{Relation to Other Frameworks}\label{sec:relation}

This section positions TSI with respect to neighbouring approaches to multi-agent interaction, irreducibility, and distributed computation. Table~\ref{tab:relations} highlights the objects each framework studies and identifies empirical comparisons that distinguish their predictions from those of TSI.

\begin{table}[H]
\caption{Conceptual relations. Each row is an empirical comparison target, not an identification.\label{tab:relations}}
\begin{tabularx}{\textwidth}{>{\raggedright\arraybackslash}p{3.2cm}>{\raggedright\arraybackslash}X>{\raggedright\arraybackslash}X}
\toprule
\textbf{Framework} & \textbf{Primary object} & \textbf{Relation to TSI} \\
\midrule
Partial information decomposition~\cite{williams2010} & Joint information unavailable from marginals & TSI represents interaction geometry and compatibility constraints; it assumes no information-theoretic decomposition. A direct empirical comparison is whether PID synergy and $\beta^n$ dissociate on tasks with matched information but different interaction structure. \\
\midrule
Cooperative game theory~\cite{shapley1953} & Coalition value function & TSI requires no transferable utility and compares no coalition values; it identifies sheaf-relative obstructions. A direct empirical comparison is whether coalition superadditivity and TSI signatures dissociate across matched task families. \\
\midrule
Integrated information~\cite{tononi2004,oizumi2014} & Axiomatic integration measure & Both address irreducibility with different primitives; no equivalence is asserted. Shared benchmark tasks could support direct comparison where the corresponding integration measure is well-defined. \\
\midrule
Sheaf-theoretic contextuality~\cite{abramsky2011,abramsky2012cohomology} & Compatibility of local measurement sections with a global section & Shares the local-to-global obstruction pattern, but TSI uses cellular interaction complexes and task-specific capacity sheaves and evaluates system capability separately through registered interventions. \\
\midrule
Sheaf neural networks~\cite{bodnar2022,hansen2019} & Learnable sheaves and sheaf Laplacians & One way to learn capacity sheaves from data (Section~\ref{sec:recipes}, Recipe B). \\
\midrule
Combinatorial topology of distributed computing~\cite{herlihy2013} & Task solvability via subdivisions & Tools for turning Proposition~\ref{prop:support} into model-specific round and bit lower bounds. \\
\bottomrule
\end{tabularx}
\end{table}

\section{Measurement and Evaluation Protocol}\label{sec:protocol}

\subsection{Procedure}

\begin{enumerate}
\item \textbf{Specify the task and interaction semantics.} Define outputs, success criteria, and the communication budget. Apply the operational simplex criterion of Definition~\ref{def:opsimplex}: distinguish pairwise relay from genuine joint interaction, with the inclusion rule written down before data collection.
\item \textbf{Construct or learn the capacity sheaf} using a recipe from Section~\ref{sec:recipes}, and \textbf{freeze it} (pre-register the construction pipeline) before any behavioural evaluation.
\item \textbf{Compute the signature.} Assemble the coboundary matrices and compute
\begin{equation}
\beta^n(K;\mathcal{F}) = \dim\ker(\delta^n) - \operatorname{rank}(\delta^{n-1}).
\end{equation}
With inner products on cochains, the sheaf Hodge Laplacian
\begin{equation}
L^n_{\mathcal{F}} = (\delta^n)^{*}\delta^n + \delta^{n-1}(\delta^{n-1})^{*}
\end{equation}
satisfies $\ker L^n_{\mathcal{F}} \cong H^n(K;\mathcal{F})$~\cite{hansen2019}. Once inner products are chosen, the harmonic representative describes how the obstruction is distributed across cells. It need not be sparse or local; localisation requires a separate criterion or regulariser. These computations reduce to ranks and null spaces of the cochain matrices and are polynomial in matrix size.
\item \textbf{Test class-indexed predictions.} Compare systems with matched agents, parameters, and communication budgets but different interaction complexes. Use task families from Proposition~\ref{prop:correspondence} and matched controls whose classes vanish. Any performance difference should be specific to the class-carrying tasks.
\item \textbf{Run causal ablations per Definition~\ref{def:operational}.} Class-targeted interventions versus class-preserving controls of matched magnitude.
\item \textbf{Report null results.} A nonzero signature with no advantage on its own correspondence tasks is evidence against the capacity interpretation of that sheaf, even when the algebra is correct.
\end{enumerate}

\subsection{Three Recipes for Capacity Sheaves}\label{sec:recipes}

If $\mathcal{F}$ is fitted after the results are known, it can be chosen to produce a desired signature. The following recipes constrain its construction.

\textbf{Recipe A: consistency sheaves from shared variables.} For each agent $v$, let $\mathcal{F}(v)$ be its state or plan space. For each simplex $\sigma$, let $\mathcal{F}(\sigma)$ contain the variables on which the agents of $\sigma$ must agree, such as shared observables, interface variables, or common resources. The map $\rho_{v\sigma}$ extracts agent $v$'s report of those variables. This follows established sensor-integration constructions~\cite{hansen2019,robinson2014,robinson2017}: $H^0$ contains globally consistent joint states, while nonzero $H^1$ can encode compatibility obstructions for cocycle-level inconsistency data. For noisy or approximate assignments, Robinson's consistency radius and consistency filtration provide complementary quantitative measures of disagreement and its localisation~\cite{robinson2020}. The system's interface specification fixes the sheaf before evaluation.

\textbf{Recipe B: learned sheaves via sheaf diffusion.} Parameterise the restriction maps and fit them by minimising sheaf Dirichlet energy on logged traces of \textit{successful} coordination episodes, as in neural sheaf diffusion~\cite{bodnar2022}. Freeze the maps before computing cohomology and evaluating held-out task instances. Fitting the sheaf on evaluation episodes invalidates the TSI claim.

\textbf{Recipe C: belief-fusion sheaves.} This recipe is the one that makes the sheaf-alteration intervention of Definition~\ref{def:operational} (2) physically realisable, since its controls act on $F$ rather than on $K$. The full Gaussian information form $(\Lambda, \eta)$ cannot itself be a stalk because positive-definite precision matrices form an open cone rather than a vector space. One compatible construction fixes $\Lambda_v \succ 0$ as registered model parameters and uses $\mathcal{F}(v) = \mathbb{R}^d$ for the information vectors $\eta_v$, with joint-observation spaces on the simplices and linear observation maps as restrictions. A second uses the ambient vector space $\mathrm{Sym}_d(\mathbb{R}) \times \mathbb{R}^d$ as the stalk and imposes $\Lambda \succ 0$ separately as a feasibility constraint. In either case, changing a restriction map corresponds to a physical intervention such as sensor recalibration, so the controls in Definition~\ref{def:operational} can be implemented in hardware or simulation.

\section{Fixed-Task Benchmarks and Exhaustive Simulations}\label{sec:benchmark}

We evaluate the framework on two finite experiments that can be enumerated exactly. In each experiment, every system receives the same \emph{task model}; only the system interaction structure and registered primitives change. If the task model also changed between systems, the comparison would mix a new problem, a new cohomology group, and a new solver. Section~\ref{sec:experiment} studies a degree-1 cycle task, and Section~\ref{sec:deg2} gives the degree-2 tetrahedral-boundary analogue.

\textbf{Fixed task model.} The task model is $(K_T,\mathcal{F}_T)=(C_m,\underline{\mathbb{R}})$ for all systems. Write the cyclically oriented edges as $e_i=[i,i{+}1\!\!\pmod m]$, fix an integer bound $B\in\mathbb{N}$ with $B\geq 1$, and let $c_i\in\{-B,\ldots,B\}\subset\mathbb{Z}$ be the datum on $e_i$. With this orientation,
\[
\ell_m(c)=\sum_{i=0}^{m-1}c_i.
\]
Hence $\beta^1(K_T;\mathcal{F}_T)=1$, the critical support is every edge, and realizability is $\ell_m(c)=0$. The experiment evaluates every point in $\{-B,\ldots,B\}^m$ under the uniform distribution. Since $K_T$ is a graph, every 1-cochain is a cocycle; $\ell_m(c)=0$ is the realizability condition. A real-valued variant could instead use the promise problem $|\ell_m(c)|\leq\varepsilon$ versus $|\ell_m(c)|\geq2\varepsilon$.

\textbf{Input allocation and outputs.} Agent $i$ initially owns $c_i$, and agent $0$ is the designated coordinator for the primary outputs. The benchmark tasks are:
\begin{itemize}
\item $\mathsf{T}_1$ (decision): agent $0$ outputs whether $[c]=0$;
\item $\mathsf{T}_3$ (designated-edge, minimum-support repair): when $[c]\neq0$, agent $0$ outputs the correction $d=-\ell_m(c)$ to the designated edge $e_0$. The correction has the smallest possible nonzero support, namely one edge. It need not minimize norm, magnitude, or application cost.
\end{itemize}
The realization task $\mathsf{T}_2$ from Proposition~\ref{prop:correspondence} remains part of the formal family. We omit it from the one-round comparison because an all-agent realization generally requires $\Theta(m\log(mB))$ output bits or a richer personalised-output primitive.

\textbf{Communication model $\Pi$.} A system is a pair $(K_S,\Pi)$ over the same $m$ agents, distinct from the fixed task model $K_T$. Communication is synchronous. Let
\[
h_B=\left\lceil\log_2(2B+1)\right\rceil,
\qquad
q_{m,B}=\left\lceil\log_2(2mB+1)\right\rceil.
\]
The registered primitives are:
\begin{itemize}
\item \emph{Pairwise channel} $\{u,v\}$: in one round, each endpoint may send an arbitrary message of at most $b$ bits to the other endpoint.
\item \emph{Joint channel} on $\sigma\subseteq A$: an atomic aggregation primitive whose function is fixed before evaluation. In one round every member of $\sigma$ receives the same $q_{m,B}$-bit value $g_\sigma$ of the members' current inputs. The value is common knowledge within $\sigma$ and counts against each recipient's congestion budget. We register $g_A(c)=\sum_{i=0}^{m-1}c_i=\ell_m(c)$ in advance as an additive reduction. This function, rather than the simplex alone, is what gives the joint channel its operational meaning under Definition~\ref{def:complex}.
\item \emph{Congestion budget}: every agent may send and receive at most $\beta$ bits in total per round across all available communication primitives; the received value of a joint channel is included in this total. Registered instances satisfy
\[
q_{m,B}\leq\beta<(m-1)h_B,
\]
so the joint reduction is admissible while an exact one-round pairwise collection at the coordinator is congested. Such conditions can be found in applications like over-the-air computation \cite{goldenbaum2013}.
\end{itemize}
Every recipient is charged for each invoked joint primitive's $q_{m,B}$-bit output. For every deterministic protocol, the subset of available primitives that it invokes is fixed as part of the registered protocol before the input is revealed; invocation decisions are not input-dependent. Only invoked outputs are charged. The joint primitive is an \emph{atomic aggregation service}: access to its participants' current inputs is part of its registered semantics. The benchmark evaluates the pre-registered reduction $g_A$ under the stated output-width and congestion cost.

The communication model is CONGEST-inspired: it combines synchronous rounds and bounded pairwise messages with an aggregate per-agent bandwidth budget~\cite{peleg2000distributed}. Proposition~\ref{prop:one-round} is a communication-complexity argument~\cite{kushilevitz1997}. TSI identifies the task class and the registered primitive predicted to remove the bottleneck; the lower bound itself comes from the communication model.

\begin{Proposition}[One-round pairwise bottleneck]\label{prop:one-round}
Fix $m\geq3$. In the coordinator-output model above, no deterministic one-round pairwise protocol can solve $\mathsf{T}_1$ whenever $\beta<(m-1)h_B$. This remains true when every non-coordinator has a direct channel to agent $0$.
\end{Proposition}
\begin{proof}
Consider any non-coordinator agent $i\in\{1,\ldots,m-1\}$. If its one-round message to agent $0$ were the same for two distinct values $x,x'\in\{-B,\ldots,B\}$, choose an index $j\in\{1,\ldots,m-1\}\setminus\{i\}$ and set $c_j=-x$, all remaining coordinates (including $c_0$) to zero. The two inputs obtained by setting $c_i=x$ and $c_i=x'$ give identical messages to the coordinator, because every other local input is unchanged, yet the first has $\ell_m(c)=0$ and the second has $\ell_m(c)=x'-x\neq0$. Hence agent $i$'s message must distinguish all $2B+1$ possible values and therefore contains at least $h_B$ bits. This holds for each of the $m-1$ non-coordinator agents. In one synchronous round their information can reach agent $0$ only through messages incident to agent $0$, which require at least $(m-1)h_B$ received bits in total. The congestion bound rules this out.
\end{proof}

\textbf{Systems.} All systems use the same agents, input allocation, and external per-agent congestion budget. For the primary comparison, let
\[
K_0=\operatorname{skel}_1(\Delta^{m-1})
\]
be the complete pairwise 1-skeleton on the $m$ agents.
\begin{itemize}
\item $S_{\mathrm{pair}}$: pairwise channels on $K_0$ only, with no higher-dimensional simplex. The outer-cycle task class extends to this graph: take a representative on $C_m$ and set it to zero on the added chords. Because $K_0$ is one-dimensional, this remains a cocycle whose restriction is the original nonzero class.
\item $S_{\mathrm{fill}}$: the same complete pairwise channels plus one joint channel on all $m$ agents. Its induced interaction complex is the full simplex $\Delta^{m-1}$, equipped with the canonical constant-sheaf extension. The restriction map
\[
H^1(\Delta^{m-1};\underline{\mathbb{R}})=0
\longrightarrow H^1(C_m;\underline{\mathbb{R}})
\]
has zero image, so the task class does not extend and the joint primitive is class-filling. The registered reduction $g_A=\ell_m$ solves $\mathsf{T}_1$ and $\mathsf{T}_3$ in one round because the tasks ask for its zero-test and designated-edge repair.
\item $S_{\mathrm{ring}}$: pairwise ring channels only, with induced complex $C_m$. This is a base-topology sanity check, not the primary ablation control.
\end{itemize}

\textbf{Same-1-skeleton primitive-closure ablation.} Relative to the registered system base $K_0$, the all-agent joint primitive in $S_{\mathrm{fill}}$ generates
\[
\mathsf{P}_{A}(K_0)=\{\tau\subseteq A:\ |\tau|\geq3\}.
\]
The class-targeted ablation removes the joint primitive and every higher-order face in this closure, leaving the complete pairwise 1-skeleton unchanged. The resulting system is $S_{\mathrm{pair}}$, and the task class again extends to its one-dimensional complex. This is the full primitive-closure removal required by Definition~\ref{def:operational}; pairwise communication is identical on both sides of the comparison.

\subsection{Degree-1 Exhaustive Simulation}\label{sec:experiment}

We instantiate $m=6$, $B=2$, so $h_B=3$, $q_{m,B}=5$, and $\beta=5\in[q_{m,B},(m-1)h_B)$. We set $b=\beta=5$ and enumerate all $5^6=15{,}625$ inputs; $1{,}751$ ($11.21\%$) are realizable and $13{,}874$ are unrealizable. The deterministic simulation and exhaustive optimiser are provided in \texttt{benchmark\_sim.py} in the versioned code repository~\cite{distributedcognition2026}. Execution instructions and expected results are provided in the repository README.

The filled system computes $g_A$ within budget and solves $\mathsf{T}_1$ and $\mathsf{T}_3$ exactly. In $S_{\mathrm{pair}}$, only messages received by agent $0$ can affect its output after one synchronous round. Edges among non-coordinators do not change that view. A deterministic $b_i$-bit message from agent $i$ partitions $\{-2,-1,0,1,2\}$ into at most $2^{b_i}$ blocks. For any fixed tuple of partitions, the optimal coordinator chooses the majority label for $\mathsf{T}_1$ and the modal correction for $\mathsf{T}_3$ in each induced view. Refinement cannot reduce either optimum. We therefore search maximal partitions for the undominated receive-bit allocations summing to $\beta=5$: $(1,1,1,1,1)$, $(2,1,1,1)$, $(2,2,1)$, $(3,1,1)$, and $(3,2)$. Allocations granting any single sender more than $h_B=3$ bits, such as $(4,1)$ and $(5)$, are dominated and need not be searched: three bits already distinguish all $2B+1=5$ values of that sender's input, so additional bits cannot refine its message partition and can instead be allocated to another sender, as in $(3,2)$. This search covers every deterministic one-round complete-pairwise protocol under the registered receive budget. Table~\ref{tab:benchmark} reports the exact optima.

\begin{table}[H]
\caption{Exhaustive evaluation for $m=6$, $B=2$, and $\beta=5$. $\mathsf{T}_1$ is evaluated on all $15{,}625$ inputs; $\mathsf{T}_3$ on the $13{,}874$ unrealizable inputs.\label{tab:benchmark}}
\centering
\begin{tabular}{lccc}
\toprule
\textbf{System} & $\boldsymbol{\mathsf{T}_1}$ & $\boldsymbol{\mathsf{T}_3}$ & \textbf{Control} \\
\midrule
$S_{\mathrm{fill}}$ (atomic $g_A$) & $1.0000$ & $1.0000$ & $1.0000$ \\
$S_{\mathrm{pair}}$ (same 1-skeleton) & $0.8880$ & $0.2663$ & $1.0000$ \\
$S_{\mathrm{ring}}$ (sanity check) & $0.8879$ & $0.1687$ & $1.0000$ \\
Majority-class baseline & $0.8879$ & --- & --- \\
\bottomrule
\end{tabular}
\end{table}

The complete-pairwise optimum decides $13{,}875$ of $15{,}625$ inputs correctly, a single input more than the always-unrealizable majority-class baseline and therefore essentially at that baseline at this budget, and returns the exact repair on $3{,}694$ of $13{,}874$ unrealizable inputs. The ring matches the baseline on $\mathsf{T}_1$ and returns $2{,}340$ exact repairs. Removing the higher-order primitive while preserving the full 1-skeleton therefore loses $1{,}750$ correct decisions and $10{,}180$ exact repairs. The control task is $c_0+c_1=0$, whose registered model has no relevant cycle class. One $h_B=3$ bit message over the edge suffices, so every system solves it exactly. The observed advantage is specific to the class-carrying tasks in this experiment.

\textbf{Benchmark results.} Proposition~\ref{prop:one-round} and the one-round solution of $S_{\mathrm{fill}}$ follow analytically from $\Pi$ and the registered additive reduction. The entries in Table~\ref{tab:benchmark} are exact finite computations: the script enumerates the input domain and all deterministic protocols in the stated search. The degree-1 experiment establishes the performance component, a class-restoring primitive-closure ablation, and task specificity through the trivial-cohomology control. Section~\ref{sec:deg2} adds a degree-2 same-2-skeleton benchmark and status-preserving stress tests; an exactly cost-matched status-preserving intervention remains part of Hypothesis~\ref{hyp:ablation}.

\textbf{Pre-registered hypotheses for learned agents.}
\begin{Hypothesis}[Learned-agent scaling]\label{hyp:cost}
For bounded learned agents, the success--round trade-off of $S_{\mathrm{ring}}$ and $S_{\mathrm{pair}}$ worsens with $m$, whereas $S_{\mathrm{fill}}$ approaches the registered one-round optimum.
\end{Hypothesis}
\begin{Hypothesis}[Specificity]\label{hyp:specificity}
$S_{\mathrm{fill}}$'s advantage over $S_{\mathrm{pair}}$ appears on $\mathsf{T}_1$ and $\mathsf{T}_3$, whose fixed model carries the class, and does not appear on matched control tasks over the same data whose registered models carry no relevant class.
\end{Hypothesis}
\begin{Hypothesis}[Ablation symmetry]\label{hyp:ablation}
Removing the full primitive closure $\mathsf{P}_{A}(K_0)$ from $S_{\mathrm{fill}}$ returns the system to the complete pairwise 1-skeleton and restores extension of the task class; for learned agents, it degrades $\mathsf{T}_1$ and $\mathsf{T}_3$ more than a matched extension-status-preserving intervention of equal registered cost.
\end{Hypothesis}

\textbf{Falsification conditions.} The learned-agent interpretation fails under the registered budgets if $S_{\mathrm{fill}}$ has no advantage on $\mathsf{T}_1$ and $\mathsf{T}_3$, if extension-status-preserving controls produce the same advantage, or if the advantage appears equally on tasks without the class. The finite protocol claim has a simpler test: a deterministic one-round pairwise protocol that matches $S_{\mathrm{fill}}$ within the stated congestion budget would contradict Proposition~\ref{prop:one-round} or expose an error in the simulation assumptions.

\subsection{A Degree-2 Benchmark with Status-Preserving Stress Tests}
\label{sec:deg2}

The experiment of Section~\ref{sec:experiment} fixes a degree-1 task model. We now repeat the design one cochain degree higher. The degree-2 benchmark is implemented in \texttt{benchmark\_deg2.py}; the accelerated three-compressor routine for the full separation window is in \texttt{fastopt3.py}~\cite{distributedcognition2026}. This has two purposes: it instantiates Proposition~\ref{prop:correspondence} in degree $2$, and it tests whether the class-filling separation persists when all three-agent interactions are retained. We also report status-preserving stress tests relevant to condition~(3) of Definition~\ref{def:operational}. These stress tests are deliberately described as directional controls rather than as a proof of full operational support, because their registered costs are not exactly matched to the class-restoring ablation in every cost dimension.

\textbf{Fixed task model.}
Let $A=\{0,1,2,3\}$ and let the task model be $(K_T,\mathcal{F}_T)=(\partial\Delta^3,\underline{\mathbb{R}})$, the boundary of the tetrahedron with the constant sheaf: all four vertices, all six edges, all four triangles, and no $3$-simplex. Then
\[
C^0\cong\mathbb{R}^4,\qquad C^1\cong\mathbb{R}^6,\qquad C^2\cong\mathbb{R}^4,\qquad C^3=0,
\qquad \operatorname{rank}\delta^0=\operatorname{rank}\delta^1=3,
\]
so $H^0\cong\mathbb{R}$, $H^1(K_T;\underline{\mathbb{R}})=0$, and
\[
H^2(\partial\Delta^3;\underline{\mathbb{R}})
\cong \mathbb{R}^4/\operatorname{im}\delta^1
\cong \mathbb{R},
\qquad
\beta^2(K_T;\mathcal{F}_T)=1.
\]
Orient the triangles as the faces of $[0,1,2,3]$. Under the standard Euclidean pairing, the annihilator
$\mathcal{A}_2=(\operatorname{im}\delta^1)^\perp$ is spanned by the total-flux functional
\begin{equation}
\ell_3(c)=c_{123}-c_{023}+c_{013}-c_{012}.
\label{eq:ell3}
\end{equation}
Equivalently, if the missing $3$-simplex is added, then evaluating $\delta^2 c$ on $[0,1,2,3]$ gives $\ell_3(c)$. Hence $\ell_3\circ\delta^1=0$. Since $K_T$ is $2$-dimensional, every $2$-cochain is a cocycle, and Proposition~\ref{prop:correspondence} gives
\[
\mathsf{R}^2(c)\ \text{solvable}
\quad\Longleftrightarrow\quad
[c]=0
\quad\Longleftrightarrow\quad
\ell_3(c)=0.
\]
By Proposition~\ref{prop:support}, the critical support $E_2^*$ is all four $2$-simplices.

This is the degree-2 counterpart of Section~\ref{sec:triangle}: there, the task class lies on the boundary of a triangle; here it lies on the boundary of a tetrahedron. Adding $[0,1,2,3]$ gives $\delta^2(c)([0,1,2,3])=\ell_3(c)$, so $\ker\delta^2=\operatorname{im}\delta^1$ and $H^2(\Delta^3;\underline{\mathbb{R}})=0$. The corresponding four-agent system primitive is therefore class-filling for this task class.

\textbf{Task family.}
The unknowns now live on the edges. Each edge $e=\{i,j\}$ carries a bilateral parameter $x_{ij}\in\mathcal{F}_T(e)$, and each triangle $\sigma$ carries a prescribed \emph{closure defect} $c_\sigma$. In the additive constant-sheaf model, $(\delta^1x)(\sigma)$ is the oriented sum of the three edge parameters around $\sigma$. The realization task $\mathsf{R}^2(c)$ asks for edge parameters satisfying $\delta^1x=c$, and $\mathsf{D}^2(c)$ asks whether the prescribed triangle-defect profile is realizable.

Algebraically, realizability requires the oriented sum of the four face defects to vanish, $\ell_3(c)=0$. This is the condition used by the benchmark. For the constant sheaf it also has a standard discrete-exterior-calculus interpretation: the unknown $1$-cochain $x$ is a discrete vector potential on the edges, the prescribed $2$-cochain $c$ assigns a flux to each face, and $\delta^1x=c$ is the discrete analogue of a curl equation on a triangulated surface. In this analogy, $\ell_3(c)=0$ is the corresponding discrete Stokes compatibility condition: the oriented total flux through the closed surface $\partial\Delta^3\simeq S^2$ vanishes. The critical support is every face because $\ell_3$ depends on all four face defects, and the registered four-agent primitive evaluates this sum in one step. The degree-1 family of Section~\ref{sec:farkas} is the analogous algebraic construction one degree lower, with unknowns on vertices and prescribed data on edges.

Beyond this discrete-calculus reading, application-specific interpretations of the edge variables and triangle defects require a domain-specific sheaf and are not assumed by the benchmark.

\textbf{Input allocation and outputs.}
Fix $B\in\mathbb{N}$. Each agent $v$ holds a private share $s_{v,\sigma}\in\{-B,\ldots,B\}$ for every triangle $\sigma\ni v$, and
\[
c_\sigma=\sum_{v\in\sigma}s_{v,\sigma},\qquad
\ell_3(c)=\sum_{v\in A}a_v,\qquad
a_v:=\sum_{\sigma\ni v}\varepsilon_\sigma s_{v,\sigma},
\]
where $\varepsilon_\sigma\in\{\pm1\}$ is the sign of $\sigma$ in \eqref{eq:ell3}. The shares make the triangle data genuinely distributed. Whether the three shares of a triangle can be aggregated atomically is a property of the system interaction model $K_S$, not of the fixed task complex $K_T$. Agent $0$ is the coordinator. We use the degree-2 analogues
\begin{itemize}
\item $\mathsf{T}^{(2)}_1$ (decision): agent $0$ outputs whether $[c]=0$;
\item $\mathsf{T}^{(2)}_3$ (designated-simplex, minimum-support repair): when $[c]\neq0$, agent $0$ outputs the correction to a fixed designated $2$-simplex $\sigma_0$. If $\varepsilon_{\sigma_0}\in\{\pm1\}$ is its coefficient in $\ell_3$, the required correction is $d=-\varepsilon_{\sigma_0}\ell_3(c)$. Thus a correction supported on one $2$-simplex always exists, and computing it is equivalent to recovering $\ell_3(c)$ exactly.
\end{itemize}

\textbf{Communication model.}
The model follows Section~\ref{sec:experiment}: one synchronous round, deterministic protocols, coordinator output, and a per-agent congestion budget. As in the degree-1 instantiation, we set the pairwise per-channel cap to $b=\beta$. Let
\[
\begin{aligned}
h_B &= \left\lceil\log_2(2B+1)\right\rceil, \qquad
w_B = \left\lceil\log_2(6B+1)\right\rceil,\\
q_\sigma &= \left\lceil\log_2(6B+1)\right\rceil, \qquad
q_A = \left\lceil\log_2(24B+1)\right\rceil.
\end{aligned}
\]
The registered primitives are pairwise channels, three-agent primitives $g_\sigma=c_\sigma$ of output width $q_\sigma$, and an all-agent primitive $g_A=\ell_3(c)$ of output width $q_A$. For every deterministic protocol, the subset of available primitives that it invokes is fixed as part of the registered protocol before the input is revealed; invocation decisions are not input-dependent. Only invoked outputs are charged. In one round every message and every primitive output is a simultaneous function of the initial shares; an agent therefore cannot receive a primitive output and forward it within the same round.

\textbf{Systems.}
All systems have the same agents, input allocation, and per-agent budget $\beta$. For the primary class-restoring ablation the registered base is $K_0=\partial\Delta^3$.
\begin{itemize}
\item $S_{\mathrm{fill}}$: complete pairwise channels, the four three-agent primitives, and the all-agent primitive. Its induced complex is $\Delta^3$ with the canonical constant-sheaf extension, and
\[
\jmath^*:H^2(\Delta^3;\underline{\mathbb{R}})=0
\longrightarrow H^2(K_T;\underline{\mathbb{R}})\cong\mathbb{R}
\]
has zero image, so the all-agent primitive is class-filling.
\item $S_{\mathrm{face}}$: the same system with the all-agent primitive removed, i.e.\ the full primitive-closure ablation
\[
\mathsf{P}_A(K_0)=\{\tau\subseteq A:|\tau|\geq4\}=\{[0,1,2,3]\}.
\]
Its induced complex is $\partial\Delta^3=K_T$, so the class extends and the ablation is class-restoring. Panel~(\textbf{c}) of Figure~\ref{fig:classfilling} depicts this configuration. This is a \emph{same-2-skeleton} ablation: all four three-agent primitives are retained.
\item $S_{\mathrm{pair}}$: pairwise channels only, with induced complex $\operatorname{skel}_1(\Delta^3)$. This is a base-topology sanity check rather than the primary ablation control.
\end{itemize}

\begin{Proposition}[One-round bottleneck in degree 2]
\label{prop:deg2bottleneck}
In the coordinator-output model above, no deterministic one-round protocol of $S_{\mathrm{pair}}$ solves $\mathsf{T}^{(2)}_1$ when $\beta<3w_B$. The same bound holds for $S_{\mathrm{face}}$, despite its four additional registered three-agent primitives.
\end{Proposition}

\begin{proof}
For $S_{\mathrm{pair}}$, the argument of Proposition~\ref{prop:one-round} applies with $a_v$ in place of $c_i$. If agent $i\neq0$ sent the same message for two local inputs with $a_i\neq a_i'$, the remaining independent shares can be completed so that the first instance has $\ell_3=0$; the second then has $\ell_3=a_i'-a_i\neq0$ while the coordinator's view is unchanged. Hence each of the three non-coordinators must distinguish all $6B+1$ possible values of $a_i$ and therefore send at least $w_B$ bits. Exactness requires at least $3w_B$ received bits at agent $0$.

For $S_{\mathrm{face}}$, only the three triangle primitives containing agent $0$ can affect the coordinator in one round. Let $T$ be the subset of those primitives that are invoked, and let $R=K_{T,2}\setminus T$. The coordinator knows the signed triangle contributions for $\sigma\in T$ exactly. For each $v\neq0$, the residual statistic is
\[
b_v=\sum_{\substack{\sigma\in R\\ v\in\sigma}}\varepsilon_\sigma s_{v,\sigma},
\]
a signed sum of $k_v=|\{\sigma\in R:v\in\sigma\}|$ independent shares. The same fooling argument forces at least $\lceil\log_2(2Bk_v+1)\rceil$ bits from agent $v$. Thus every exact protocol has receive cost at least
\[
q_\sigma|T|+\sum_{v=1}^{3}\left\lceil\log_2(2Bk_v+1)\right\rceil.
\]
It remains to minimise this over the eight subsets $T\subseteq\operatorname{St}(0)$. Each $v\neq0$ lies in three triangles, two of which contain agent $0$, so $k_v=3-|T\cap\operatorname{St}(v)|$. Writing $u_B=\lceil\log_2(4B+1)\rceil$ and using $q_\sigma=w_B$, the four cases up to symmetry give
\[
\begin{aligned}
|T|=0:\ &k=(3,3,3), &\text{cost} &= 3w_B, \\
|T|=1:\ &k=(2,2,3), &\text{cost} &= 2w_B+2u_B, \\
|T|=2:\ &k=(1,2,2), &\text{cost} &= 2w_B+h_B+2u_B, \\
|T|=3:\ &k=(1,1,1), &\text{cost} &= 3w_B+3h_B.
\end{aligned}
\]
Since $6B+1<2(4B+1)$ we have $w_B\leq u_B+1$, and $u_B\geq3$ for every $B\geq1$; hence the $|T|=1$ cost exceeds the $|T|=0$ cost by $2u_B-w_B\geq u_B-1\geq2>0$. The $|T|=2$ cost exceeds the $|T|=1$ cost by $h_B>0$, and the $|T|=3$ cost exceeds the $|T|=0$ cost by $3h_B>0$. The minimum is therefore attained at $T=\varnothing$ and equals $3w_B$: invoking any of the three-agent outputs available to the coordinator strictly increases the exact-protocol lower bound.
\end{proof}

Since $S_{\mathrm{fill}}$ needs only the $q_A$-bit output of $g_A$, the registered separation window is
\[
q_A\leq\beta<3w_B.
\]

\textbf{Exhaustive evaluation.}
We instantiate $B=1$, so $h_B=2$, $w_B=q_\sigma=3$, and $q_A=5$; the registered separation window is $5\leq\beta<9$. The input domain is $\{-1,0,1\}^{12}$, i.e.\ $3^{12}=531{,}441$ inputs, of which $73{,}789$ ($13.88\%$) are realizable and $457{,}652$ are unrealizable. We exhaustively optimise all four integer budgets $\beta\in\{5,6,7,8\}$ in this window.

The search is exhaustive over all deterministic one-round coordinator-output protocols by the following reduction. Write $U$ for the statistic agent $0$ knows exactly after its chosen primitive invocations, so that
\[
\ell_3 \;=\; U+\sum_{v\neq0}b_v ,
\]
with $b_v$ as in the proof of Proposition~\ref{prop:deg2bottleneck}. The cochains $U$ and $b_1,b_2,b_3$ are signed sums over pairwise disjoint sets of shares; since the shares are drawn independently and uniformly, these four statistics are mutually independent. Both tasks depend on agent $v$'s input only through $b_v$, so a message that separates two inputs with the same $b_v$ refines the coordinator's view without changing the conditional distribution of $\ell_3$; an optimal deterministic message is therefore, without loss of generality, a function of $b_v$ alone, i.e.\ a partition of the attainable values of $b_v$ into at most $2^{b}$ blocks for a $b$-bit message. For any fixed tuple of partitions the optimal coordinator takes the majority label for $\mathsf{T}^{(2)}_1$ and the modal value for $\mathsf{T}^{(2)}_3$ in each induced view, and refining a partition cannot lower either optimum; granting a sender more bits likewise cannot lower it. It therefore suffices to search the maximal partitions of each $b_v$-support over the undominated receive-bit allocations, exactly as in Section~\ref{sec:experiment}.

At $\beta=5$ and $\beta=6$, some undominated allocations leave three senders simultaneously compressed. To make these cases tractable without changing the protocol class, the optimiser exploits the additivity of the objective over message-block triples: the contribution of each distinct block triple is computed once and then gathered over all maximal partition triples. This is an implementation-level acceleration of the same exhaustive partition search. It reproduces the direct optimiser's previously reported optima at $\beta=7,8$ and completes the two remaining budgets. Randomised protocols are outside this search, as they are in degree~1. Table~\ref{tab:deg2} reports the resulting exact optima; $\mathsf{T}^{(2)}_1$ and $\mathsf{T}^{(2)}_3$ are optimised separately.

\begin{table}[H]
\caption{Exhaustive evaluation across the complete registered separation window for the degree-2 task model $(\partial\Delta^3,\underline{\mathbb{R}})$ with $B=1$. Each system entry reports $(\mathsf{T}^{(2)}_1,\mathsf{T}^{(2)}_3)$; $\mathsf{T}^{(2)}_1$ is evaluated on all $531{,}441$ inputs and $\mathsf{T}^{(2)}_3$ on the $457{,}652$ unrealizable inputs. The trivial-cohomology control is solved exactly by every system at every reported budget.\label{tab:deg2}}
\centering
\begin{tabular}{ccccc}
\toprule
$\boldsymbol{\beta}$ & $\boldsymbol{S_{\mathrm{fill}}}$ & $\boldsymbol{S_{\mathrm{face}}}$ & $\boldsymbol{S_{\mathrm{pair}}}$ & \textbf{Majority baseline $\mathsf{T}^{(2)}_1$} \\
\midrule
$5$ & $1.0000/1.0000$ & $0.8884/0.4666$ & $0.8884/0.4666$ & $0.8612$ \\
$6$ & $1.0000/1.0000$ & $0.9295/0.6299$ & $0.9295/0.6299$ & $0.8612$ \\
$7$ & $1.0000/1.0000$ & $0.9453/0.7193$ & $0.9453/0.7193$ & $0.8612$ \\
$8$ & $1.0000/1.0000$ & $0.9681/0.8401$ & $0.9681/0.8401$ & $0.8612$ \\
\bottomrule
\end{tabular}
\end{table}

For $\beta=5,6,7,8$, respectively, the optimum of $S_{\mathrm{face}}$ decides $472{,}149$, $493{,}965$, $502{,}361$, and $514{,}509$ of $531{,}441$ inputs correctly, and returns the exact repair on $213{,}563$, $288{,}257$, $329{,}211$, and $384{,}488$ of $457{,}652$ unrealizable inputs. At every budget the optimum invokes none of the three three-agent primitives whose outputs are visible to the coordinator. The fourth face primitive, on $\{1,2,3\}$, cannot affect the coordinator's output within the same simultaneous round. Forcing one coordinator-visible face primitive yields $(0.8612,0.2790)$, $(0.8672,0.3841)$, $(0.8763,0.4581)$, and $(0.8944,0.5362)$ across $\beta=5,6,7,8$, respectively; forcing two is infeasible at $\beta=5$ and yields $(0.8612,0.2386)$, $(0.8612,0.2748)$, and $(0.8612,0.3383)$ at $\beta=6,7,8$.

Thus, throughout the complete registered separation window, the exhaustive optimum of $S_{\mathrm{face}}$ is attained without using any coordinator-visible three-agent primitive and is identical to the optimum of $S_{\mathrm{pair}}$. This equality is not a rounding artefact: every $S_{\mathrm{pair}}$ protocol is available to $S_{\mathrm{face}}$, so the latter can never score lower. Proposition~\ref{prop:deg2bottleneck} predicts the failure of exact protocols without the class-filling primitive, and the exhaustive search shows that the additional local three-agent services also fail to improve the optimal approximate performance anywhere in the registered window.

Relative to $S_{\mathrm{pair}}$, the system $S_{\mathrm{face}}$ adds a larger catalogue of local joint operations: four three-agent primitives, each of width $q_\sigma=3$, for a total registered output width of $4q_\sigma=12$ bits across the four services. This total is not a fully matched capacity measure, because the primitives have different recipient sets and only invoked outputs consume congestion. Nevertheless, throughout the registered separation window the exhaustive optimum of $S_{\mathrm{face}}$ is identical to that of $S_{\mathrm{pair}}$. By contrast, adding the single $q_A=5$-bit all-agent primitive makes both tasks exact. Thus, within this registered model, merely adding the larger catalogue of local three-agent primitives does not improve these tasks; the successful primitive is specifically the one whose induced $3$-simplex is class-filling. This is evidence about the semantics and placement of the registered primitives and the extension status of the class, not a general claim that the quantity of joint capacity is irrelevant.

\textbf{Control task.}
The trivial-cohomology control uses the full $2$-simplex $\{0,1,2\}$ with the task ``is $c_{012}=0$?''. For the constant sheaf on this simplex, $\delta^1$ is surjective and $H^2=0$. With the same share semantics, two $h_B=2$-bit messages from agents $1$ and $2$ to agent $0$ suffice, so every system solves the control exactly for $\beta\geq4$.

\textbf{Status-preserving robustness checks.}
Condition~(3) of Definition~\ref{def:operational} asks for a status-preserving intervention matched to the class-restoring intervention in registered cost. The present finite experiment includes two useful robustness checks, but they are not matched controls.

\textbf{C1 (service removal).} Disable two standalone three-agent service primitives while retaining the all-agent primitive $g_A$. Under the primitive-closure convention of Definition~\ref{def:operational}, retaining $g_A$ keeps the induced interaction complex equal to $\Delta^3$, so the extension status is unchanged. The exact solver uses only $g_A$, and performance remains $(1.0000,1.0000)$.

\textbf{C2 (budget reduction).} Reduce the per-agent congestion budget from $\beta=8$ to $\beta=5$. The interaction complex and extension status are unchanged, and $g_A$ still fits in budget, so performance remains $(1.0000,1.0000)$.

Both interventions remove registered capacity without changing extension status and leave the exact solution intact. However, C1 and C2 alter different cost coordinates from the $S_{\mathrm{fill}}\to S_{\mathrm{face}}$ ablation. They therefore provide status-preserving robustness evidence, not the exactly cost-matched control required to conclude condition~(3).

\textbf{Scope of the evidence.}
Proposition~\ref{prop:deg2bottleneck} and the one-round solution of $S_{\mathrm{fill}}$ follow analytically from the registered communication model and additive reduction; the entries of Table~\ref{tab:deg2} are exact finite computations over the full input domain for all four integer budgets in the registered separation window. Together with Section~\ref{sec:experiment}, the degree-2 experiment establishes a second class-specific finite-model separation, now under a same-2-skeleton class-restoring ablation, and provides status-preserving robustness evidence rather than a matched-control result. It does not by itself establish full operational support under Definition~\ref{def:operational}, because an exactly cost-matched status-preserving intervention is still absent. As in degree 1, the finite computations do not measure learned-agent behaviour.

\textbf{The general pattern.}
For every $n\geq1$, the constant-sheaf task model $(\partial\Delta^{n+1},\underline{\mathbb{R}})$ has $\beta^n=1$. Its top-degree cohomology is detected by the oriented facet-sum functional
\[
\ell^{(n)}(c)=\sum_{i=0}^{n+1}(-1)^i c\!\left(d_i[0,\ldots,n+1]\right),
\]
whose support is all $n$-simplices. Adding the missing $(n+1)$-simplex kills this top-degree class. Section~\ref{sec:triangle} is the $n=1$ boundary-simplex instance of this pattern; the cycle benchmark of Section~\ref{sec:experiment} is a broader degree-1 analogue, while the present section gives the $n=2$ boundary-of-a-tetrahedron instance.

A further degree-2 observation is specific to invertible rank-one restrictions. On $\partial\Delta^3\simeq S^2$, any rank-one local system over $\mathbb{R}$ is trivial because $\pi_1(S^2)=0$. Thus the monodromy mechanism of Section~\ref{sec:twisted}, which can remove the degree-1 class on a cycle, has no nontrivial rank-one local-system analogue on the tetrahedral boundary. This statement does not cover general capacity sheaves with noninvertible or higher-rank restriction maps, whose cohomology can change for other reasons.

\section{Limitations and Research Agenda}

The present analysis has six limits. \textbf{(1) Linearity.} Linear cellular sheaves approximate systems that may be nonlinear, stochastic, adaptive, and history-dependent. Possible extensions include sheaves of nonlinear spaces and persistence summaries over parameterised linearisations. \textbf{(2) Model identification.} The interaction complex is inferred rather than observed directly. Treating message traffic as evidence of higher-order interaction can create spurious simplices and hence spurious signatures; the operational simplex criterion in Definition~\ref{def:opsimplex} reduces this risk but cannot remove it. \textbf{(3) Primitive semantics.} In the benchmarks, class-filling means that a registered atomic reduction is available at a stated cost. Changing the function or cost model, or charging the primitive's internal input collection under a different communication model, may change the prediction. \textbf{(4) Descriptive scope.} Proposition~\ref{prop:correspondence} gives an exact algebraic criterion for the linear cochain-realization family in every degree, and Section~\ref{sec:deg2} instantiates degree 2. It does not imply that an arbitrary higher-order multi-agent task has the same semantics; domain-specific reductions and nonlinear task families remain open. \textbf{(5) Operational-support control.} Section~\ref{sec:deg2} adds status-preserving stress tests, but their registered costs are not exactly matched to the class-restoring ablation in every cost dimension. Full operational support under condition~(3) of Definition~\ref{def:operational} therefore remains a hypothesis rather than a conclusion of the finite benchmarks. \textbf{(6) Static, undirected structure.} Directed, asynchronous, or changing interactions require a different construction.

The reduction library should cover multi-sensor fusion, cooperative planning, and general linear constraints, with task-specific meanings for the dual cohomological witnesses. Capacity sheaves should also be learned from traces and then tested on held-out tasks. On the communication side, Propositions~\ref{prop:one-round} and~\ref{prop:deg2bottleneck} should be extended to multi-round and bit-complexity bounds in explicit models using tools from~\cite{herlihy2013}. Other open problems include domain-specific higher-degree correspondence tasks beyond the linear boundary-simplex construction and exactly cost-matched status-preserving controls. The benchmarks also permit direct comparisons with PID, coalition values, and integration measures.

\textbf{Multiscale composition.} Nested organisations suggest a multiscale model in which a base complex encodes inter-team interaction and fibres encode intra-team interaction. We do not develop this here for two reasons: we have no registered construction that produces a fibration-like map $p:E\to B$ from interaction complexes satisfying the standard hypotheses, and the Leray--Serre spectral sequence computes ordinary cohomology of the total space, so a capacity-sheaf version would require a Leray-type spectral sequence for cellular sheaf cohomology over the base. We record only what the constant-coefficient case would say if such a construction were available. Under the standard hypotheses~\cite{mccleary2001},
\[
E_2^{p,q}=H^p\!\big(B;\mathcal{H}^q(X;\mathbb{R})\big)\Longrightarrow H^{p+q}(E;\mathbb{R}).
\]
A nonzero $E_2^{1,1}$, recording degree-1 structure at both the team and organisation scales, need not give a nonzero class in $H^2(E;\mathbb{R})$. In this first-quadrant spectral sequence, the only differential that can affect position $(1,1)$ is $d_2:E_2^{1,1}\to E_2^{3,0}$: for $r\geq3$ the outgoing differential lands in negative vertical degree and the incoming one starts in negative horizontal degree. Hence $E_3^{1,1}=E_\infty^{1,1}$, and the possible failure of cross-scale composition is concentrated in $d_2$. Identifying registered constructions for which $d_2$ vanishes, and a sheaf-level analogue of the statement, are open problems.

\section{Conclusions}

TSI provides a way to turn a cohomological obstruction into a testable hypothesis about multi-agent interaction. The central distinction is simple: the \emph{task model} carries the structural signature, while the \emph{system} provides behavioural evidence of capability. In general, a nonzero $H^n(K;\mathcal{F})$ identifies a sheaf-relative obstruction at cochain degree $n$ under the registered interaction structure and compatibility maps.

For the linear cochain-realization tasks studied here, Proposition~\ref{prop:correspondence} shows in every degree that a cocycle datum is realizable precisely when its cohomology class vanishes; equivalently, nonzero classes are exactly the task-relative irreducible cocycle demands of Definition~\ref{def:irreducibility}. Degree 1 has the direct agent-state interpretation developed through distributed difference constraints; Section~\ref{sec:deg2} provides a degree-2 realization on the tetrahedral boundary. The worked examples also show why the sheaf is essential: changing $\mathcal{F}$ can remove a class from a cycle or create one even when the underlying complex is contractible. The signature must therefore be defined for the full pair $(K,\mathcal{F})$, not for topology alone.

The fixed-task benchmarks illustrate how the structural claim can be tested operationally. In degree 1, the registered all-agent additive primitive solves both class-carrying tasks exactly while optimal pairwise systems remain limited under the same-1-skeleton ablation. In degree 2, the analogous four-agent primitive solves the tetrahedral-boundary tasks exactly while the same-2-skeleton system remains limited throughout the registered separation window $5\leq\beta<9$. The degree-2 experiment sharpens the comparison in one respect: the ablated system retains every three-agent primitive, yet throughout that entire window its exhaustive optimum coincides with that of pairwise-only communication. Those local services form a larger catalogue of registered joint operations, but their aggregate output width is not a fully matched capacity measure because recipients differ and only invoked outputs are charged. The result therefore does not show that the quantity of joint capacity is generally irrelevant; it shows that these status-preserving local primitives do not improve the registered tasks, whereas the class-filling four-agent primitive does. Both benchmarks include trivial-cohomology controls, and the degree-2 experiment adds status-preserving robustness checks. These are class-specific finite-model separations under registered one-round communication models; they do not yet establish full operational support under an exactly cost-matched status-preserving intervention.

The main contribution of TSI is a disciplined way to state such questions: define the task and its capacity sheaf first, identify the relevant cohomological obstruction, register the interaction primitive and its cost, and then test whether performance changes specifically when that obstruction is filled or restored.

\section*{Statements and Declarations}
\subsection*{Funding}
This research received no external funding.

\subsection*{Competing interests}
The authors declare no conflict of interest.

\subsection*{Author contributions}
Steven Motta and Alessandro Rozza contributed to conceptualization, methodology, and formal analysis. Steven Motta, Davide Mambelli, and Alessandro Rozza contributed to writing the original draft and to reviewing and editing the manuscript. All authors have read and approved the manuscript.

\subsection*{Data availability}
All benchmark inputs are synthetic and are generated by exhaustive enumeration of the finite domains specified in Sections~\ref{sec:experiment} and~\ref{sec:deg2}; no external dataset is required. The benchmark inputs and result counts can be regenerated using the publicly available code identified in the Code availability statement.

\subsection*{Code availability}
The benchmark implementation is publicly available in the TSI simulation repository~\cite{distributedcognition2026}. The version used for this submission is commit \nolinkurl{0dfdc90fcf69e585a5192e7b4e9a5b91ed4594d2}.

\end{document}